\documentclass[12pt]{l4dc2021}

\newcommand{\junk}[1] {}
\usepackage[font=small,labelfont=bf]{caption}
\usepackage{booktabs}  
\usepackage{siunitx}
\title[Neural-Newton Solvers]{Accelerating Dynamical System Simulations with\\Contracting and Physics-Projected Neural-Newton Solvers}
\usepackage{times}
\author{\Name{Samuel Chevalier$^1$} \Email{schev@elektro.dtu.dk}\\
 \Name{Jochen Stiasny$^1$} \Email{jbest@elektro.dtu.dk}\\
 \Name{Spyros Chatzivasileiadis$^1$} \Email{spchatz@elektro.dtu.dk}\\
 \addr $^1$ Technical University of Denmark,
  Kgs. Lyngby, Denmark}

\begin{document}

\maketitle

\begin{abstract}%
Recent advances in deep learning have allowed neural networks (NNs) to successfully replace traditional numerical solvers in many applications, thus enabling impressive computing gains. One such application is time domain simulation, which is indispensable for the design, analysis and operation of many engineering systems. Simulating dynamical systems with implicit Newton-based solvers is a computationally heavy task, as it requires the solution of a parameterized system of differential and algebraic equations at each time step. A variety of NN-based methodologies have been shown to successfully approximate the trajectories computed by numerical solvers at a fraction of the time. However, few previous works have used NNs to model the numerical solver itself. For the express purpose of accelerating time domain simulation speeds, this paper proposes and explores two complementary alternatives for modeling numerical solvers. First, we use a NN to mimic the linear transformation provided by the inverse Jacobian in a single Newton step. Using this procedure, we evaluate and project the exact, physics-based residual error onto the NN mapping, thus leaving physics ``in the loop''. The resulting tool, termed the Physics-pRojected Neural-Newton Solver (\texttt{PRoNNS}), is able to achieve an extremely high degree of numerical accuracy at speeds which were observed to be up to 31\% faster than a Newton-based solver. In the second approach, we model the Newton solver at the heart of an implicit Runge-Kutta integrator as a \textit{contracting map} iteratively seeking a fixed point on a time domain trajectory. The associated recurrent NN simulation tool, termed the Contracting Neural-Newton Solver (\texttt{CoNNS}), is embedded with training constraints (via CVXPY Layers) which guarantee the mapping provided by the NN satisfies the Banach fixed-point theorem; successive passes through the NN are therefore guaranteed to converge to a unique, fixed point. Explicitly capturing the contracting nature of Newton iterations leads to significantly increased NN accuracy relative to a vanilla NN. We test and evaluate the merits of both \texttt{PRoNNS} and \texttt{CoNNS} on three dynamical test systems.
\end{abstract}


\section{Introduction}
Across many applications, Neural Networks (NNs) are being increasingly constructed, and subsequently utilized, as iterative self-mapping functions~\citep{Chen:2018}:
\begin{align}\label{eq: self_map}
    x^{(i+1)}=f(x^{(i)},\theta),
\end{align}
where the NN mapping ${\ f}:{\mathbb R}^n\rightarrow {\mathbb R}^n$ is parameterized by some input $\theta$. Residual Networks (ResNets)~\citep{He:2016}, normalizing flows~\citep{Kobyzev:2020}, and even standard Recurrent Neural Networks (RNNs)~\citep{Yu:2019} all approximately leverage this self-mapping structure\citep{Chen:2018}. The observation that (\ref{eq: self_map}) can be interpreted as a forward Euler step \citep{Lu:2018,Haber:2017} of some underlying continuous dynamical system has lead to the re-emergence~\citep{Massaroli:2020,Cohen:1983} of research which treats NNs as dynamical systems~\citep{Li:2019}. Notable recent works include, for example, Neural ODEs~\citep{Chen:2018} with its many variants~\citep{Massaroli:2020,Tzen:2019,Poli:2019,Zhang:2019}, and Deep Equilibrium Models~\citep{Bai:2019,Bai:2020}. These approaches, collectively referred to as Continuous Deep Learning Models (CDLMs), provide a variety of computational and modeling benefits, but they are also able to directly leverage a century worth of tools from the differential equations community.

The CDLM perspective arises by treating (\ref{eq: self_map}) as a discretized Euler step of some continuous system; in the limit, (\ref{eq: self_map}) transforms into a true derivative (i.e., $\dot{{x}}=f({x},\theta)$), and CDLMs naturally emerge~\citep{Chen:2018}. In this paper, we offer an alternative perspective: rather than having (\ref{eq: self_map}) (or its continuous counterpart) directly model the evolution of some system of interest, we instead use it to model the iterative steps taken by the numerical \textit{solver} of the system. Learning the behaviour of a numerical solver (rather than the behaviour of a system itself) can be advantageous for a number of reasons, including the learned solver's ability to provide arbitrary levels of simulation accuracy. %and (ii) the ability to model the evolution of a system with a smaller sized NN for a desired level of accuracy. 
In this paper, we focus exclusively on implicit Runge-Kutta (IRK) integration methods which require the solution of a (potentially nonlinear) system of equations at each time step; such IRK methods are typically solved with a Newton-based root-finding tool and are therefore computationally expensive~\citep{Iserles:2008}. NN-based models can be used to help reduce this computational effort. Modeling Newton's method with a NN has two salient alternatives; one can use a NN $\Phi$ to either (i) predict the linear transformation provided by the inverted Jacobian inside of a Newton loop, or (ii) directly predict iterative state updates via $x^{(i+1)}\leftarrow\Phi(x^{(i)})$. In this paper, we outline novel frameworks for both of these approaches, and we compare and contrast their merits.

\textbf{Predicting linear transformations:} We first use a NN to predict the expensive linear transformation (i.e., Jacobian inverse $J^{-1}$) mappings which are at the computational heart of a Newton routine. In doing so, we exploit the fact that Newton steps can include some degree of error and still iteratively arrive at the desired root. In this paper, we use a NN to learn these linear transformations, and then we iteratively project the residual error of future time domain trajectory points onto these linear transformations. We thus refer to the resulting tool as the Physics-pRojected Neural-Newton Solver (\texttt{PRoNNS}), since it projects physics-based residuals directly into the NN at each iteration.

\textbf{Directly predicting state updates:} We also explore removing the physics-based residual function altogether to learn a fully data-driven iterative solver. By directly modeling a numerical solver, we can equip (\ref{eq: self_map}) with a particular mathematical property that few dynamical systems actually possess: contraction. A self-mapping system contracts if a distinct set of inputs are farther apart than the outputs to which they map~\citep{Pata:2019}. Thus, we use (\ref{eq: self_map}) to model a contracting numerical solver. If certain conditions are satisfied (e.g., conditions for the Newton-Kantorovich Theorem~\citep{Ortega:2000}), then Newton's method is guaranteed to contract around the fixed point it is seeking. By ensuring that the Newton solver always iterates in a locally contracting region, we can train a NN which both mimics Newton's iterative contracting behavior and, furthermore, is guaranteed to converge. The resulting NN is termed the Contracting Neural-Newton Solver (\texttt{CoNNS}).

% directly predict state updates via $x^{(i+1)}\leftarrow\Phi(x^{(i)})$, or (ii) predict the linear transformation provided by the inverted Jacobian inside of a Newton loop via $x^{(i+1)}\leftarrow x^{(i)}-\Phi(r(x^{(i+1)}))$. In this paper, we outline novel frameworks for both of these approaches, and we compare and contrast their merits.

%\begin{figure}
%\centering
%\includegraphics[width=1\linewidth]{Newton_CoNNS_PRoNNS}
%\caption{Shown is the typical Newton iteration in panel (a), with the primary computational bottleneck highlighted in red. In panel (b), the Physics-pRojected Neural Newton Solver replaces the inverted Jacobian with a NN mapping $\Phi$. }
%\label{fig:Newton_Update_CoNNS}
%\end{figure}

\textbf{Paper Contributions:} In this paper, we design two approaches for modeling IRK numerical solvers. Through \texttt{PRoNNS}, (i) we develop a model which learns the linear transformations required at each Newton solver step, and (ii) we overcome the problem of training on residuals whose norms decay to 0 by implementing a normalization procedure which preserves the mappings of the underlying linear transformations. Through \texttt{CoNNS}, (iii) we derive conditions under which a NN of the form (\ref{eq: self_map}) is guaranteed to contract and thus satisfy the Banach fixed point theorem, and (iv) we pose these contraction conditions as simple, semidefinite programming constraints which can be implemented with open source tools (CVXPY Layers~\citep{Agrawal:2019}) inside of a NN training routine. Finally, (v) we provide a comparison between these alternative numerical solver modeling approaches and showcase their capabilities on several dynamical test systems.

%; once trained with these constraints' (v) we showcase \texttt{CoNNS}' ability to mimic an IRk solver more effectivly than a vanilla NN.

\section{Related work}
The direct modeling of dynamical systems using NNs has a rich history~\citep{Garcia:1998,Milano:2002,Kosmatopoulos:1995,Tiumentsev:2019}; subsequent simulation of the resulting parameterized models, though, requires the use of classical numerical solvers. Recently, attention has shifted towards constructing models which are capable of \textit{directly} predicting the time domain response of a dynamical system, thus bypassing the need for a numerical solver. In particular, the so-called Physics Informed Neural Network (PINN)~\citep{Raissi:2019}, which uses physics-based sensitivities to regularize the training loss function, has achieved great success in extrapolating the solutions of Partial Differential Equations (PDEs), both in continuous and discrete time applications. A variety of follow up works~\citep{Pang:2019,Yeonjong:2020,Lu:2021,Misyris:2020} have quickly improved upon the method, adding additional training constraints and extending it into new domains. In particular,~\citep{Roehrl:2020} formally extends PINN modeling to ODEs; ODEs are the primary focus of this paper.

Rather than directly parameterizing a NN with an input time variable $t$, which is often necessary with PINNs, an alternative approximation scheme uses a \textit{flow map} interpretation of dynamical systems~\citep{Ying:2006}; in this case, a NN-based flow map function learns to directly advance a dynamical system state from ${ x}(t)$ to ${x}(t+\Delta t)$. This is accomplished with a feedforward NN in~\citep{Pan:2018} and a ResNet in~\citep{Qin:2019}. The standard ResNet approach is improved upon in~\citep{Yuying:2020}, where a hierarchy of ResNets is used in order overcomes the problem of numerical stiffness in dynamical systems. NNs which mimic flow maps, however, must implicitly model both the dynamics of a system and the application of a numerical integration tool~\citep{Luchtenburg:2014} (e.g., implicit Runge-Kutta) all at once, and there is no way to control the accuracy of the prediction they provide (other than increased training).

Ultimately, both PINNs and flow maps directly learn the trajectories associated with dynamical systems. Such trajectories, however, are usually computed with implicit numerical solvers. An even more fundamental approach to learning dynamical trajectories is to therefore directly learn the \textit{iterative steps} taken by the associated numerical solver. While this approach has not been applied to the problem of time domain simulation, to the authors' knowledge, it has been applied in limited cases to other engineering problems~\citep{Baker:2020,Mathia:1995}. However, system physics has never been incorporated inside the iterative solver loop (as we do with \texttt{PRoNNS}), and the stability of the NN-based solver has not been explicitly addressed (as we consider with \texttt{CoNNS}).

The convergence of our second solver, \texttt{CoNNS}, coincides with the convergence of its RNN model. There is extensive work which analyzes the stability of RNNs, beginning with~\citep{Simard:1988}. Rigorous energy function formulations are offered in~\citep{Cao:2005,Mostafa:2011,Hu:2002,Yi:2013}. When deriving \texttt{CoNNS}, we target contraction, which is a stronger notion than stability, but has sparsely been addressed in the discrete RNN literature until more recently. \citep{Steck:1992} was one of the first papers which derived sufficient conditions for the contraction of a single layer RNN with bounded activation functions. Subsequently,~\citep{Mandic:2000} developed analytical conditions related to the slope of the sigmoid activation function and the size of the weighting parameters. Most recently,~\citep{Miller:2018} investigated the stability of RNNs from the perspective of so-called $\lambda$-contractive sets. These methods are improved upon and analyzed in~\citep{Revay:2020}, where projected gradient descent is used to project the NN model into a contracting space. In both~\citep{Miller:2018} and~\citep{Revay:2020}, however, contraction of the RNN is utilized only as a conservative proxy to achieve stability; in neither case is convergence to a unique fixed point an explicit goal for the underlying RNN.

\section{Mathematical background}
%In this section, we provide mathematical background associated with useful matrix properties, contraction analysis, and time domain integration.

We denote $\varsigma_m(\cdot)$ as the maximum singular value operator, i.e., $\varsigma_{m}({A})={\rm max}\{\sqrt{\lambda({A}{A}^{T})}\}$ \citep{Horn_1990}. Notably, the largest singular value provides the bound $\Vert{A}{x}\Vert_2 \le\varsigma_{m}({A})\Vert {x}\Vert_2$.

\subsection{Implicit Runge-Kutta Integration}
We consider a time-invariant system of nonlinear ODEs given as $\dot{{x}}(t)\triangleq\tfrac{d}{dt}{{x}}(t)={f}({x}(t))$.
%\begin{align}\label{eq: ODE_LC}
    %\dot{{x}}(t)\triangleq\frac{d}{dt}{{x}}(t)={f}({x}(t)).
%\end{align}
The goal of time domain simulation~\citep{Yuying:2020} is to integrate ${x}(t)$ forward in time via
\begin{align}\label{eq: TDI}
{x}(t+\Delta t)={x}(t)+\int_{t}^{t+\Delta t}{f}({x}(\tau))d\tau.
\end{align}
Since the closed form solution of (\ref{eq: TDI}) is rarely available, $x(t)$ is typically integrated using a Runge-Kutta time-stepping approach. Depending on how the parameters of the integration scheme are chosen, Runge-Kutta methods can take the form of many popular integrators (e.g., backward Euler); these methods can be classified as either explicit or implicit. In this work, we exclusively consider implicit integration schemes, since they are commonly used in many physics-based applications; they are, however, computationally expensive to solve. In the Implicit Runge-Kutta (IRK) method~\citep{Iserles:2008}, the future state ${x}(t+\Delta t)$ is written as the sum
\begin{align}\label{eq: RK1}
{x}(t+\Delta t)={x}(t)+\Delta t \sum_{i=1}^{s}b_{i}{k}_{i},
\end{align}
where $s$ is the number of ``stages'' associated with the IRK method, and ${k}_{i}$ is a vector of trajectory derivatives at various points $t+c_i \Delta t$ between $t$ and $t + \Delta t$. These derivative terms are computed as
\begin{align}\label{eq: RK2}
{k}_{i}={f}\bigg({x}(t)+\Delta t\sum_{j=1}^{s}\alpha_{i,j}{k}_{j}\bigg), \quad i\in \{1,\ldots,s\}.
\end{align}
The parameters $\alpha_{i,j}$, $b_i$, and $c_i$ are given coefficients from the Butcher tableau. Notably, the Runge-Kutta step given by (\ref{eq: RK1})-(\ref{eq: RK2}) represents an implicit, nonlinear system of equations.\junk{This can be seen by writing the system of equations from (\ref{eq: RK2}) via
\begin{align}\label{eq:K_RK}
{r}({k}_{1},{k}_{2},....,{k}_{s})=\left\{ \begin{array}{c}
{ 0}={k}_{1}-{f}({x}(t)+\Delta t(\alpha_{1,1}{k}_{1}+\alpha_{1,2}{k}_{2}+\cdots\alpha_{1,s}{k}_{s}))\\
{0}={k}_{2}-{f}({x}(t)+\Delta t(\alpha_{2,1}{k}_{1}+\alpha_{2,2}{k}_{2}+\cdots\alpha_{2,s}{k}_{s}))\\[-2pt]
\vdots\\
{0}={k}_{s}-{f}({x}(t)+\Delta t(\alpha_{s,1}{k}_{1}+\alpha_{s,2}{k}_{2}+\cdots\alpha_{s,s}{k}_{s})).
\end{array}\right.
\end{align}With TI, the system (\ref{eq:K_RK}) collapses down to}
Typically, an iterative root finding tool, such as Newton's method, is used to drive this system of nonlinear equations to $0$ at each time step. Without loss of generality, this paper will focus exclusively on the trapezoidal integration method~\citep{Iserles:2008} as a guiding example, since trapezoidal integration is the primary workhorse behind many ODE solvers~\citep{Milano:2010}. With trapezoidal integration, the implicit system of equations associated with (\ref{eq: RK2}) is given by
\begin{align}\label{eq: Trap}
{0}	={k}_{2}-{f}({x}(t)+\tfrac{\Delta t}{2}{k}_{1}+\tfrac{\Delta t}{2}{k}_{2}))\triangleq r({ k}_2),
\end{align}
where ${k}_{1}={f}({x}(t))$. System (\ref{eq: Trap}) is typically solved by linearizing $r({k}_2)\approx r({k}_2^{(0)})+{J}({k}_2^{(0)})\Delta{k}_2$, setting the expression equal to $0$, and then defining the iterative self-map
\begin{align}\label{eq: NewtRK}
{k}^{(i+1)}_2={k}_2^{(i)}-{J}({k}_2^{(i)})^{-1}r({k}_2^{(i)})\triangleq{g}({k}_2^{(i)}).
\end{align}
The primary computational bottleneck associated with IRK integration comes from solving nonlinear system (\ref{eq: Trap}) by iterating on (\ref{eq: NewtRK}). For convenience, ${k}\triangleq{k}_2$ in the remainder of this paper.

\subsection{Fixed points of contracting systems}
\begin{definition}[Contraction Mapping~\citep{Pata:2019}]\label{Def: cont} A function ${f}:\mathbb{R}^{n}\rightarrow\mathbb{R}^{n}$ is said to be contracting, or a contraction mapping, if, for any ${x},{y}\in \mathbb{R}^{n}$, there exists $0\le \mu < 1$ such that
\begin{align}\label{eq: cont_cond}
\left\Vert {f}({x})-{f}({y})\right\Vert_2 \le\mu\left\Vert {x}-{y}\right\Vert_2.
\end{align}
\end{definition}
% Using this definition of contraction, we recall the celebrated Banach Fixed-Point Theorem (FPT).

% \footnote{Contraction mappings are often defined with respect to some metric space $\mathcal X$ equipped with a suitable distance operator $d(\cdot)$. In our application, however, we target the specific metric space of $\mathbb{R}^{n}$ equipped with the euclidean distance metric (i.e., the $l2$ norm $\Vert\cdot\Vert_2$).}

\begin{theorem}[Banach Fixed-Point Theorem~\citep{Pata:2019}]
Let ${f}:\mathbb{R}^{n}\rightarrow\mathbb{R}^{n}$ be contracting on the complete metric space of $\mathbb{R}^{n}$. Then, ${f}$ has a unique fixed point ${x}^*$, such that ${x}^*={f}({x}^*)$. Moreover, for any ${x}\in \mathbb{R}^n$, the sequence ${f}\circ\cdots{f}\circ{f}({x})$ converges to ${x}^*$.
\end{theorem}

\section{The Physics-pRojected Neural-Newton Solver (\texttt{PRoNNS})}
Panel (\textbf{a}) of Fig. \ref{fig:Newton_CoNNS_PRoNNS} depicts the computational bottleneck associated with Newton iterations which are at the heart of any IRK method. In order to alleviate this bottleneck, we use a ReLU-based NN to learn the linear transformation which maps a residual error vector function $r(k^{(i)})$ to a Newton step via $\Delta k = -{J}({k}^{(i)})^{-1}{r}({k}^{(i)})$. ReLU-based NNs naturally provide piecewise linear input-output mappings, so they are particularly well-suited for learning the linear transformation codified by ${J}({k}^{(i)})^{-1}$. Thus, we replace this linear transformation with ReLU-based NN $\Phi_{p}$ which takes the residual error, the current state trajectory estimate $k^{(i)}$, and the current state value $x(t)$ as inputs:
\begin{align}\label{eq: PRoNNS}
k^{(i+1)}=k^{(i)}-\Phi_{p}\left(r(k^{(i)}),k^{(i)},x(t)\right).
\end{align}
%\begin{remark}
%Once biased by $k^{(i)}$ and $x(t)$, the NN $\Phi_{p}(\cdot)$ in (\ref{eq: PRoNNS}) provides an affine ($Ax+b$) mapping between the residual error vector $r(k^{(i)})$ and the estimated step $\Delta k$.
%\end{remark}
Learning the linear mapping $J(k^{(i)})^{-1}r(k^{(i)})$ from simulated datasets is practically challenging, since the norm of the residual error vector $r$ decays to $0$ as $k$ converges to its fixed point. To overcome this challenge, we normalize $r(k^{(i)})$ by its own norm, and we additionally scale the output step by the same size. The updated Newton iterations are equivalent to 
\begin{align}\label{eq: PRoNNS_Scaled}
k^{(i+1)}=k^{(i)}-\Vert r(k^{(i)})\Vert J(k^{(i)})^{-1}\tfrac{r(k^{(i)})}{\Vert r(k^{(i)})\Vert}.
\end{align}
When training a NN to mimic (\ref{eq: PRoNNS_Scaled}), we treat $r(k^{(i)})/\Vert r(k^{(i)})\Vert$ as an input, and the output emulates $J(k^{(i)})^{-1}r(k^{(i)})/\Vert r(k^{(i)})\Vert$. We re-scale the NN output by $\Vert r(k^{(i)})\Vert$ when implementing \texttt{PRoNNS}:
\begin{align}\label{eq: PRoNNS_Scaled_NN}
k^{(i+1)}=k^{(i)}-\Vert r(k^{(i)})\Vert\Phi_{p}\left (r(k^{(i)})/\Vert r(k^{(i)})\Vert,k^{(i)},x(t)\right ).
\end{align}
Despite the input scaling in (\ref{eq: PRoNNS_Scaled_NN}), $\Phi_{p}$ is still trained to match the same linear transformation provided by the Jacobian transformation in (\ref{eq: NewtRK}). 
%In our test results, we train the NN mapping in (\ref{eq: PRoNNS_Scaled_NN}) on scaled residual Newton step data from simulated time domain trajectories.
Finally, we note that even in the presence of linear transformation prediction error, routine (\ref{eq: PRoNNS_Scaled_NN}) still has the potential to find the fixed point $k^*$. In the left-hand panel of Fig. \ref{fig:Linearizations}, we show representative steps taken by a Newton solver. In the right-hand panel, we show how steps taken by \texttt{PRoNNS}, even in the presence of gross linearization error, can still converge to the fixed point of the residual function $r(k)$. Such convergence, however, is not guaranteed. In the following section, we introduce a NN-based solver whose convergence is guaranteed.

\begin{figure}
\centering
\includegraphics[width=1\linewidth]{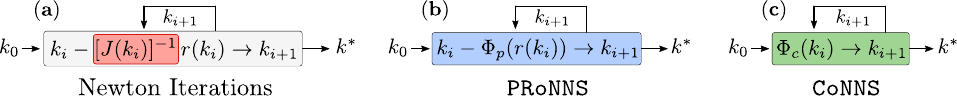}
\caption{Panel (\textbf{a}) depicts classical Newton iterations associated with (\ref{eq: NewtRK}). Panel (\textbf{b}) depicts \texttt{PRoNNS}, where the inverted Jacobian has been replaced by a piecewise linear NN mapping. Panel (\textbf{c}) depicts \texttt{CoNNS}, where a NN mapping directly iterates on state updates.}
\label{fig:Newton_CoNNS_PRoNNS}
\end{figure}

\begin{figure}
\centering
\includegraphics[width=0.75\linewidth]{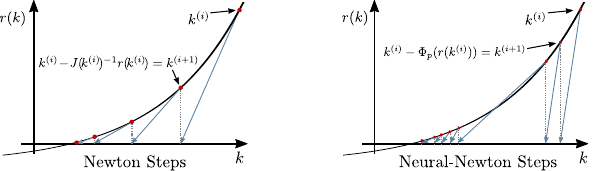}
\caption{The left panel depicts typical Newton steps, while the right panel demonstrates how imperfect Neural-Newton steps, as taken by \texttt{PRoNNS}, can still potentially converge to the desired root.}
\label{fig:Linearizations}
\end{figure}

\section{The Contracting Neural-Newton Solver (\texttt{CoNNS})}
In this section, we build a NN which effectively emulates the Newton solver routine in (\ref{eq: NewtRK}), but we apply training constraints to ensure contraction of the NN mapping. We begin by casting the Newton iterations from (\ref{eq: NewtRK}) as a contracting system which is searching for a fixed point. Generally, Newton may not contract for an arbitrarily large time step $\Delta t$ or arbitrarily chosen initial conditions. However, in this work, we assume these values are chosen such that Newton \textit{does} represent a contracting system. Contraction in some region can be guaranteed by applying conditions from the Newton-Kantorovich Theorem~\citep{Ortega:2000} or, more simply, by ensuring the gradient magnitude of the Newton map over some convex set is bounded to be less than 1. We may define analogous sufficient conditions in $n$-dimensions by taking the multivariable derivative of map (\ref{eq: NewtRK}) with respect to $k_2$ and ensuring the Jacobian singular values are bounded by unity:
\begin{align}\label{eq: Newt_cont}
\varsigma_{m}\left({J}^{-1}\left[\begin{array}{cccc}
(\frac{\partial}{\partial {k}_{2_1}}{J}){J}^{-1}{r} & (\frac{\partial}{\partial {k}_{2_2}}{J}){J}^{-1}{r} & \cdots & (\frac{\partial}{\partial {k}_{2_n}}{J}){J}^{-1}{r}\end{array}\right]\right)<1,\;\forall {k}\in {\mathcal K}.
\end{align}

%\begin{assumption}\label{eq: Assump_cont}
%When solving nonlinear system (\ref{eq: Trap}) via (\ref{eq: NewtRK}), the time step size $\Delta t$ and initialization of ${k}_2$ are always chosen apriori to ensure that (\ref{eq: Newt_cont}) is satisfied and, therefore, (\ref{eq: NewtRK}) contracts.
%\end{assumption}
% The derivation of (\ref{eq: Newt_cont}) can be found in the Appendix. Assumption \ref{eq: Assump_cont} often holds in practice for suitably small time steps. When it does not hold, the size of the time step $\Delta t$ can always be lowered.

\subsection{Conditions for guaranteed contraction of a neural network}
%Inspired by the contracting behaviour of (\ref{eq: NewtRK}), we now seek to train a feedforward NN ${\Phi}:{\mathbb R}^{2n} \rightarrow{\mathbb R}^{n}$ which mimics a Newton solver and is additionally guaranteed to have contraction characteristics. Thus, each pass through the NN will analogously represent a single iteration of the Newton solver; repeated passes will converge to a fixed point. 
We define an $h$-layer NN ${k}^{(i+1)}={\Phi}_c({k}^{(i)},{x}(t))$, equipped with ReLU activation $\sigma(\cdot)$, via
\begin{align}\label{eq: Phi_DNN}
{k}^{(i+1)}=\sigma\left({W}_{h}\ldots\sigma\left({W}_{2}\sigma\left({W}_{1}{k}^{(i)}+{U}{x}(t)+{b}_{1}\right)+{b}_{2}\right)\ldots+{b}_{h}\right) \triangleq{\Phi}_c({k}^{(i)},{x}(t)).
\end{align}
We now present conditions which are sufficient to ensure the self-mapping routine (\ref{eq: Phi_DNN}) will contract.

\begin{remark}\label{lemma: ReLU}
For all vectors ${x},{y}\in{\mathbb R}^n$, the ReLU operator $\sigma:\mathbb{R}^{n}\rightarrow\mathbb{R}^{n}$ satisfies 
\begin{equation}\label{eq: Relu_prop}
\left\Vert \sigma({x})-\sigma({y})\right\Vert_p \le\left\Vert {x}-{y}\right\Vert_p,\;\forall p\in\{1,2,...,\infty\}.
\end{equation}
%\begin{proof}
%In the scalar case, $\left|\sigma(x)-\sigma(y)\right|\le\left|x-y\right|$. Applying the $p$-norm definition,
%$\left\Vert {x}-{y}\right\Vert_{p} =\big(\sum_{i=1}^{n}\left|{x}_{i}-{y}_{i}\right|^{p}\big)^{1/p}\le\big(\sum_{i=1}^{n}\left|\sigma({x}_{i})-\sigma({y}_{i})\right|^{p}\big)^{1/p}\!\!=\left\Vert \sigma({x})-\sigma({y})\right\Vert_{p}.$
%\end{proof}
\end{remark}
%We now consider a single $i^{\rm th}$ layer of (\ref{eq: Phi_DNN}): 
%\begin{align}\label{eq: layer}
    %{x}=\sigma({W}_{i}{x}+{b}_{i})\triangleq {l}({x}).
%\end{align}

\begin{lemma}\label{lem:single_layer}
Consider the $i^{\rm th}$ layer of (\ref{eq: Phi_DNN}). If $\varsigma_m({W}_{i})<1$, then this layer is a contraction mapping.
\begin{proof}
We directly apply the definition of contraction (\ref{eq: cont_cond}) for arbitrary inputs $x$ and $y$. The result yields $\left\Vert \sigma\left({W}_{i}{x}+{b}_{i}\right)-\sigma\left({W}_{i}{y}+{b}_{i}\right)\right\Vert _{2}  \le \left\Vert {W}_{i}\left({x}-{y}\right)+{b}_{i}-{b}_{i}\right\Vert_{2} \le\varsigma_m({W}_i)\left\Vert {x}-{y}\right\Vert_{2}.$
\end{proof}
\end{lemma}
\begin{theorem}\label{thm: NN_FPT}
A sufficient condition to ensure (\ref{eq: Phi_DNN}) satisfies the Banach Fixed-Point Theorem is
\begin{align}\label{eq: sup_W1}
\sup_{i\in\{1, \ldots, h\}}\,\varsigma_{m}({W}_{i})<1.
\end{align}
%This is a sufficient condition to ensure that (\ref{eq: Phi_DNN}) satisfies the Banach Fixed-Point Theorem.
\begin{proof}
Write (\ref{eq: Phi_DNN}) as a sequence of mappings, such that ${g}_{1}({k}) =\sigma\left({W}_{1}{k}+{U}{x}(t)+{b}_{1}\right)$, ${g}_{2}({k}) =\sigma\left({W}_{2}{g}_{1}({k})+{b}_{2}\right)$, ..., ${g}_{h}({k}) =\sigma\left({W}_{h}{g}_{h-1}({k})+{b}_{h}\right)$. By Lemma \ref{lem:single_layer}, each of these functions represents a contraction if (\ref{eq: sup_W1}) is satisfied. The composition of functions ${\Phi} \triangleq{g}_{n}\circ\cdots\circ{g}_{2}\circ{g}_{1}$ which are individually contracting results in a new function which is also contracting. Thus, ${\Phi}({k},{x}(t))$ is contracting and necessarily satisfies the Banach FPT.
\end{proof}
\end{theorem}
Directly constraining the singular values of a matrix (i.e., via (\ref{eq: sup_W1})) is challenging. Instead, we seek to use semidefinite programming (SDP) tools which can efficiently constrain matrix eigenvalues. Furthermore, such tools are conveniently available through CVXPY Layers~\citep{Agrawal:2019}.
Note that $\sigma({W})<1$ implies $\lambda\left({W}{W}^{T}\right)<1$, and furthermore, ${I}-{W}{I}^{-1}{W}^{T}\succ{0}$. By the Schur complement lemma~\citep{VanAntwerp:2000,Revay:2020}
\begin{equation}\label{eq: shcur_inexact}
\left[\begin{array}{cc}
{I}(1-\epsilon) & {W}\\
{W}^{T} & {I}(1-\epsilon)
\end{array}\right]\succeq{0}\;\Rightarrow\;\varsigma_{m}({W})<1.
\end{equation}
\junk{\begin{equation}\label{eq: shcur_exact}
\left[\begin{array}{cc}
{I} & {W}\\
{W}^{T} & {I}
\end{array}\right]\succ{0}\;\Leftrightarrow\;\varsigma_{m}({W})<1
\end{equation}
with no conservativeness. To make (\ref{eq: shcur_exact}) amenable to SDP tools, we again invoke parameter $\epsilon$:
\begin{equation}\label{eq: shcur_inexact}
\left[\begin{array}{cc}
{I}(1-\epsilon) & {W}\\
{W}^{T} & {I}(1-\epsilon)
\end{array}\right]\succeq{0}\;\Rightarrow\;\varsigma_{m}({W})<1,
\end{equation}
which introduces a very small degree of conservativeness. While (\ref{eq: shcur_inexact}) is a generally less restrictive (because there is no symmetry constraint), (\ref{eq: sym_cond}) is computationally cheaper to implement.}
The contraction condition (\ref{eq: sup_W1}) can be satisfied for all square matrices by imposing (\ref{eq: shcur_inexact}) via conventional SDP tools; non-square matrices $W_1$ and $W_h$ are dealt with by defining and dealing with augmented matrices, e.g., $\tilde{{W}}_{1}=\left[{W}_{1}\,|\,{M}_1\right]\in{\mathbb R}^{n\times n}$. At each training step, CVXPY Layers is used to optimally projects the unconstrained matrix ${W}_{i}$ into a constrained space via
\begin{equation}\label{eq: min_Wi}
\begin{aligned}
\min_{\hat{{W}}_{i}}\;\; & \big\Vert {W}_{i}-\hat{{W}}_{i}\big\Vert _{F}^{2},\;\forall i\in\{1,\ldots h\},\;\;\; 
{\rm s.t.}\; (\ref{eq: shcur_inexact}).
\end{aligned}
\end{equation}
Constrained matrices $\hat{{W}}_{i}$ then replace their unconstrained counterparts ${{W}}_{i}$ in (\ref{eq: Phi_DNN}). The projected training routine in (\ref{eq: min_Wi}) guarantees that (\ref{eq: Phi_DNN}) will converge to a unique fixed point $k^*$. Panel (\textbf{c}) of Fig. \ref{fig:Newton_CoNNS_PRoNNS} portrays the functionality of \texttt{CoNNS}, which iterates until convergence.

\section{Numerical test results}\label{sec_results}
In this section, we provide test results associated with three simulated systems: a two-state cubic oscillator~\citep{Yuying:2020}, a three-state Hopf bifurcation~\citep{Yuying:2020}, and a 10-state electrical power system, known as the Kundur system~\citep{Kundur:1994}.
\begin{equation*}
\begin{array}{l}
\textbf{Cubic Oscillator:}\\
\dot{x} =-0.1x^{3}+2y^{3}\\
\dot{y} =-2x^{3}-0.1y^{3}\\
\hspace{1em}
\end{array}\;\;\bigg\vert\;\;
\begin{array}{l}
\textbf{Hopf bifurcation:}\\
\dot{\mu} =0\\
\dot{x}   =\mu x+y-x(x^{2}+y^{2})\\
\dot{y}   =\mu y-x-y(x^{2}+y^{2})
\end{array}\;\;\bigg\vert\;\;
\begin{array}{l}
\textbf{Kundur system:}\\
\dot{\delta}_{i} =\omega_{i}\\
\dot{\omega}_{i} ={\hat p}_{i}\!-\!{\hat d}_i\omega_{i}\!-\!\sum\!{\hat B}_{ij}\sin(\delta_{i}\!-\!\delta_{j})\\
\forall i \in {\mathcal I}.
\end{array}
\end{equation*}
\junk{\begin{equation*}
\left.\begin{array}{l}
\textbf{Cubic Oscillator:}\\
\dot{x} =-0.1x^{3}+2y^{3}\\
\dot{y} =-2x^{3}-0.1y^{3}
\end{array}\right\} \;\text{Cubic Oscillator}\quad\quad\left.\begin{array}{l}
\dot{\mu} =0\\
\dot{x}   =\mu x+y-x(x^{2}+y^{2})\\
\dot{y}   =\mu y-x-y(x^{2}+y^{2})
\end{array}\right\} \;\text{Hopf bifurcation}
\end{equation*}}
To collect training data, we defined a set of initial conditions associated with each system, and we randomly perturbed these conditions in order to generate strong system perturbations. We then simulated the resulting deterministic trajectories via the standard Newton-based trapezoidal integration method of (\ref{eq: Trap})-(\ref{eq: NewtRK}). We used a step size ${\Delta t}$ between 0.025s and 0.001s, and Newton convergence tolerance was set to $\epsilon=10^{-9}$. Additional settings can be found directly in the code~\citep{Code:2021}. After collecting 50 trajectories for each system, we appropriately trained both \texttt{PRoNNS} and \texttt{CoNNS} on the Newton step data. Training was performed with PyTorch~\citep{PyTorch:2019}, and we used CVXPY Layers~\citep{Agrawal:2019} to solve (\ref{eq: min_Wi}). We used Adam~\citep{Adam:2014} with learning rates set between $10^{-3}$ and $10^{-5}$ (see code). All simulation and training was performed on a Dell XPS laptop equipped with an Intel i7 CPU @ 2.60GHz 16 GB of RAM.
\subsection{Results: \texttt{PRoNNS}}

To model \texttt{PRoNNS}, we trained NNs containing three hidden layers; the NNs were respectively given 10 (cubic oscillator), 10 (Hopf), and 20 (Kundur) nodes per layer. We evaluated the performance of \texttt{PRoNNS} by testing it on 50 trajectories stemming from the same initial condition distributions as the training set. We then computed the 2-norm error across all variables and all time for every trajectory; the error associated with the $j^{\rm th}$ trajectory, therefore, is $e_{j}=\Vert\hat{x}_{i}(t)-x_{i}(t)\Vert,\forall t,\forall i$, where $\hat{x}_{i}(t)$ is the $i^{\rm th}$ state predicted by \texttt{PRoNNS}. Test results are depicted in Fig.~\ref{fig:trajectories_pronns}. Notably, \texttt{PRoNNS} was able to accurately solve for trajectories across a wide region of initial conditions and in temporal regions well beyond its training (e.g., it was trained on 45 seconds of simulation data for the cubic/Hopf systems, but it could generalize well beyond these times). Error statistics are reported in Table \ref{tbl:error_pronns}; from these data, \texttt{PRoNNS} performed almost equally well on both unseen test trajectories and on training data. In Fig.~\ref{fig:error_pronns}, we depict the error accumulated by \texttt{PRoNNS} in various regions of state space for the cubic oscillator.

\begin{figure}[ht]
    \centering
  \subfigure[\small Cubic Oscillator]{
      \includegraphics[width=0.32\linewidth]{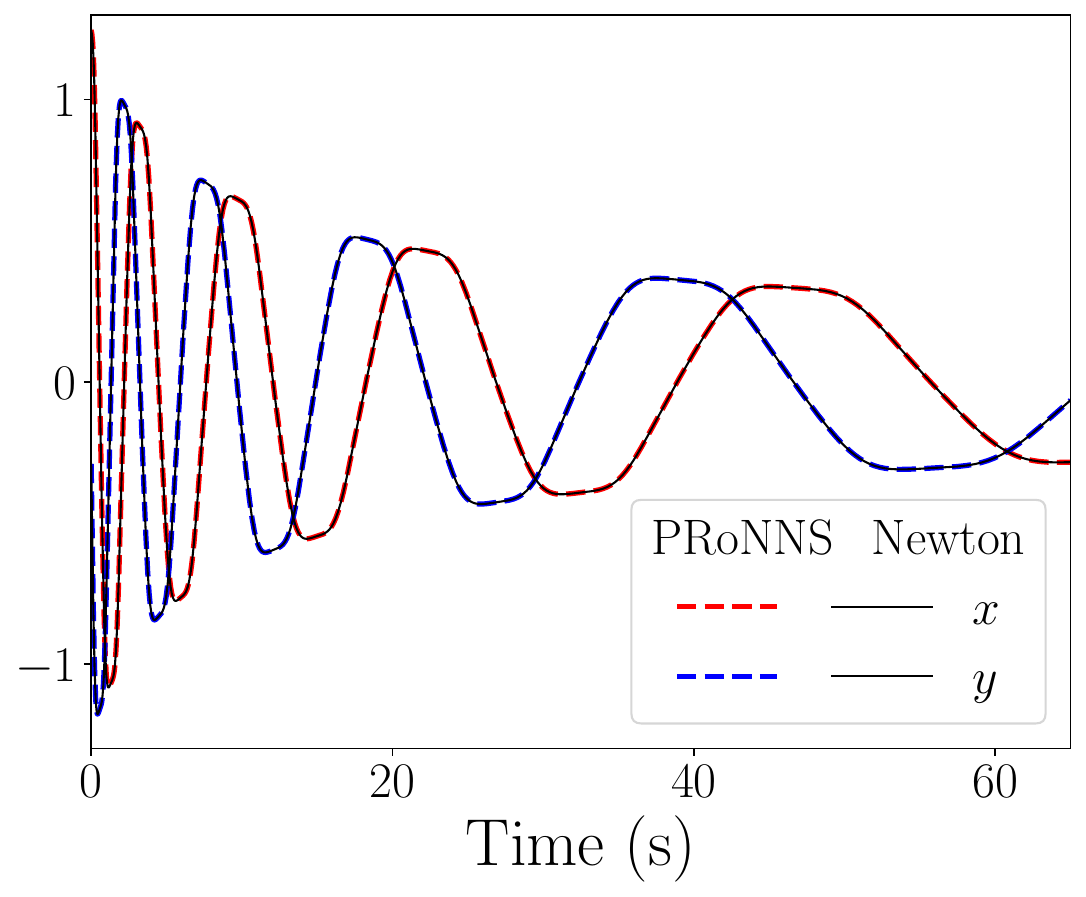}}
    \hfill
  \subfigure[\small Hopf Bifurcation]{
        \includegraphics[width=0.32\linewidth]{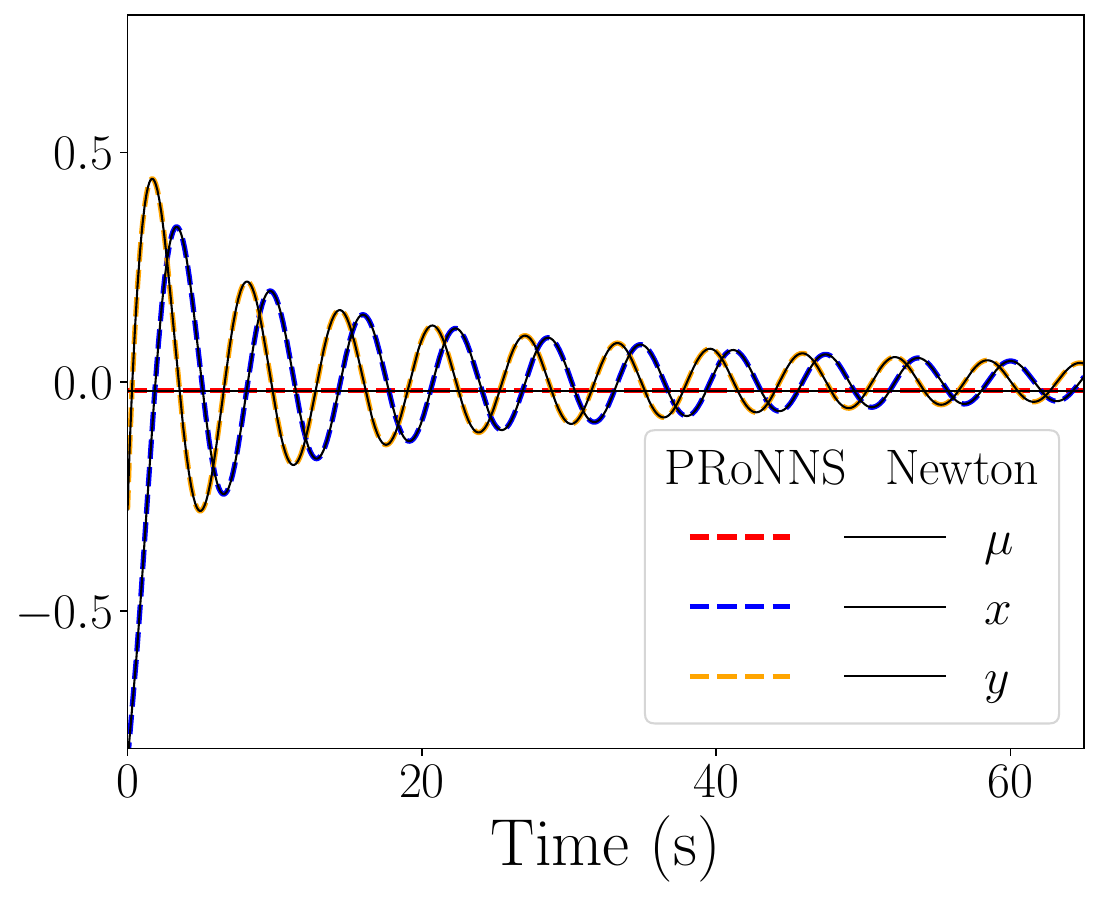}}
    \hfill
  \subfigure[\small Kundur]{
        \includegraphics[width=0.32\linewidth]{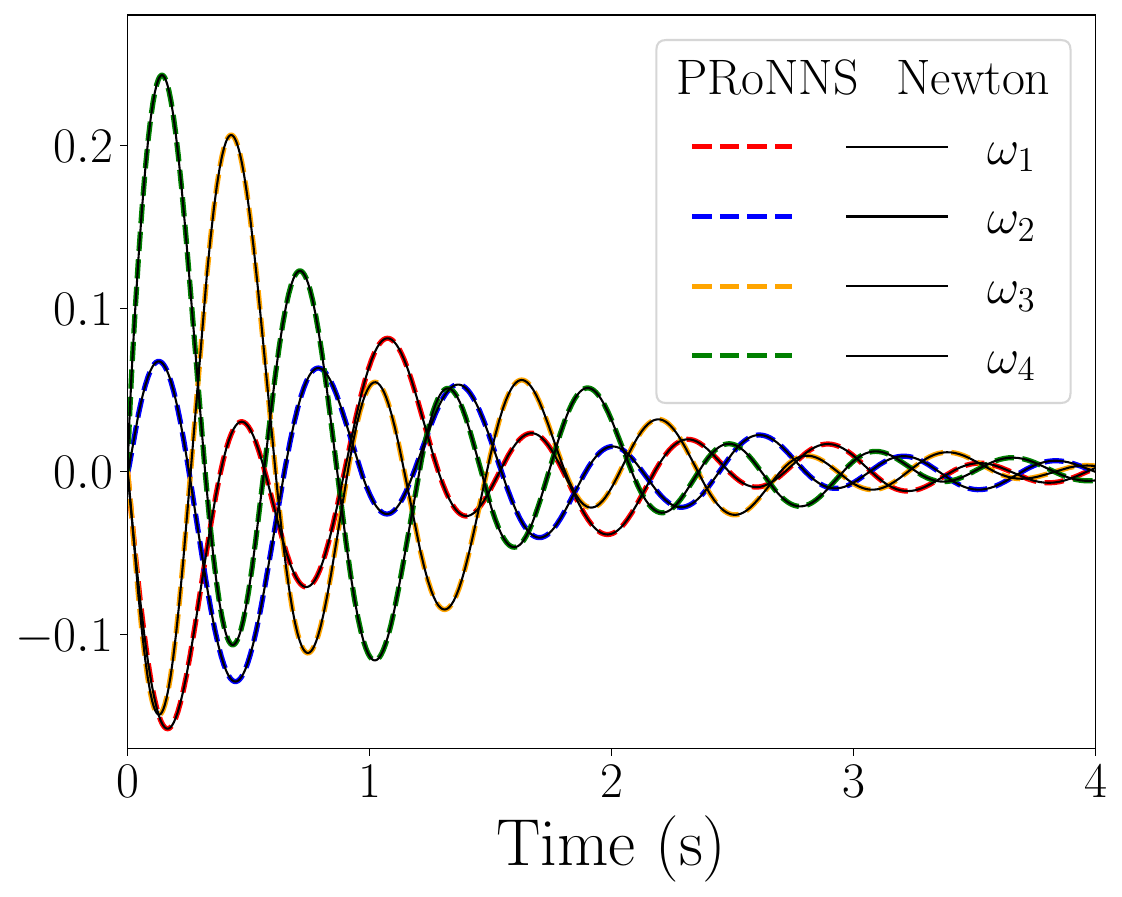}}
  \caption{Simulated trajectories pulled from test data, generated by Newton (benchmark) and \texttt{PRoNNs}.}
  \label{fig:trajectories_pronns} 
\end{figure}

\paragraph{Timing Analysis:} %Due to the small NN size, \texttt{PRoNNS} was computationally competitive with an explicit Newton solver.  summarizes our timing analysis: 
In order to obtain a first impression of the computational properties, we conducted a timed comparison of the calculation of a single time step $\Delta t$ using Newton iterations and \texttt{PRoNNs} for a range of states $x_0$, using thousands of repetitions and selecting the lowest run-time to reduce noise. For the cubic oscillator, \texttt{PRoNNS} required about 65\% \textit{more} time than the Newton iterations. For the Hopf Bifurcation system, the two methods matched each other very closely. Lastly, for the Kundur system, the largest among the test cases, \texttt{PRoNNs} required 31\% \textit{less} time than Newton; i.e., \SI{0.075}{\milli\second} for Newton's \SI{0.108}{\milli\second}. Absolute times are strongly influenced by $\Delta t$ and $x_0$ as well as the NN size, but the trend in the relative comparison matches the following consideration. Newton's method is dominated by system solve ${J}^{-1}{r}$ and scales with ${\mathcal O}(n^3)$. Meanwhile, evaluating $h$-layer NNs with $m$ nodes per layer scales with $\sim {\mathcal O}(m^2)$, since $m\gg h$. Thus, computational benefit of \texttt{PRoNNs} is realized more saliently in large systems whose size satisfies $n>m^{2/3}$. 

%The number of states $n$ is related to this scaling but not directly and we might be able to exploit this indirect relationship for computational benefits of \texttt{PRoNNs}, in particular for larger systems.

\junk{\begin{table}
  \caption{Timing Analysis of Newton vs \texttt{PRoNNS}. Time in milliseconds per evaluation step.}
  \small
  \label{tbl:timing}
  \centering
  \begin{tabular}{lccc}
    \toprule
    & \multicolumn{1}{c}{Cubic oscillator} & \multicolumn{1}{c}{Hopf bifurcation} &\multicolumn{1}{c}{Kundur}\\
    \midrule
    Newton & 0.112 & 0.112 & 0.109\\
    \texttt{PRoNNS} & 0.185 & 0.112 & 0.075\\
    \bottomrule
  \end{tabular}
\end{table}}

\begin{table}
  \caption{Training and test error, computed as the 2-norm of the trajectory approximation error.}
  \footnotesize
  \label{tbl:error_pronns}
  \centering
  \begin{tabular}{llcccccc}
    \toprule
    & & \multicolumn{2}{c}{Cubic oscillator ($\times 10^{-3}$)} & \multicolumn{3}{c}{Hopf bifurcation ($\times 10^{-4}$)} &\multicolumn{1}{c}{Kundur ($\times 10^{-4}$)}\\
    \cmidrule(lr){3-4} \cmidrule(lr){5-7} \cmidrule(lr){8-8}
    Metric  & Data   &   $\left\Vert x - \hat{x} \right\Vert _{2}$ & $\left\Vert y - \hat{y} \right\Vert _{2}$ & $\left\Vert \mu - \hat{\mu} \right\Vert _{2}$ & $\left\Vert x - \hat{x} \right\Vert _{2}$ & $\left\Vert y - \hat{y} \right\Vert _{2}$ & $\left\Vert x - \hat{x} \right\Vert _{2}$ \\
    \midrule
    mean & Training & 1.008 & 1.004 & 4.418 $\times 10^{-5}$  & 3.422 & 3.406 & 1.397\\
    % sd & Training & 0.913 & 0.913 & 2.835 $\times 10^{-5}$ & 3.062 & 3.063 & 0.877\\
    % max & Training & 5.019 & 4.921 & 11.44 $\times 10^{-5}$ & 8.799 & 8.927 & 4.486\\
    \midrule
    mean & Test & 0.978 & 0.969 & 4.456 $\times 10^{-5}$ & 3.766 & 3.755 & 1.774\\
    sd & Test & 0.652 & 0.652 & 2.846 $\times 10^{-5}$ & 3.188 & 3.155 & 1.055\\
    max & Test & 2.654 & 2.597 & 10.72 $\times 10^{-5}$ & 9.083 & 9.199 & 4.376\\
    \bottomrule
  \end{tabular}
\end{table}

\begin{figure}
    \centering
    \includegraphics[width=0.57\linewidth]{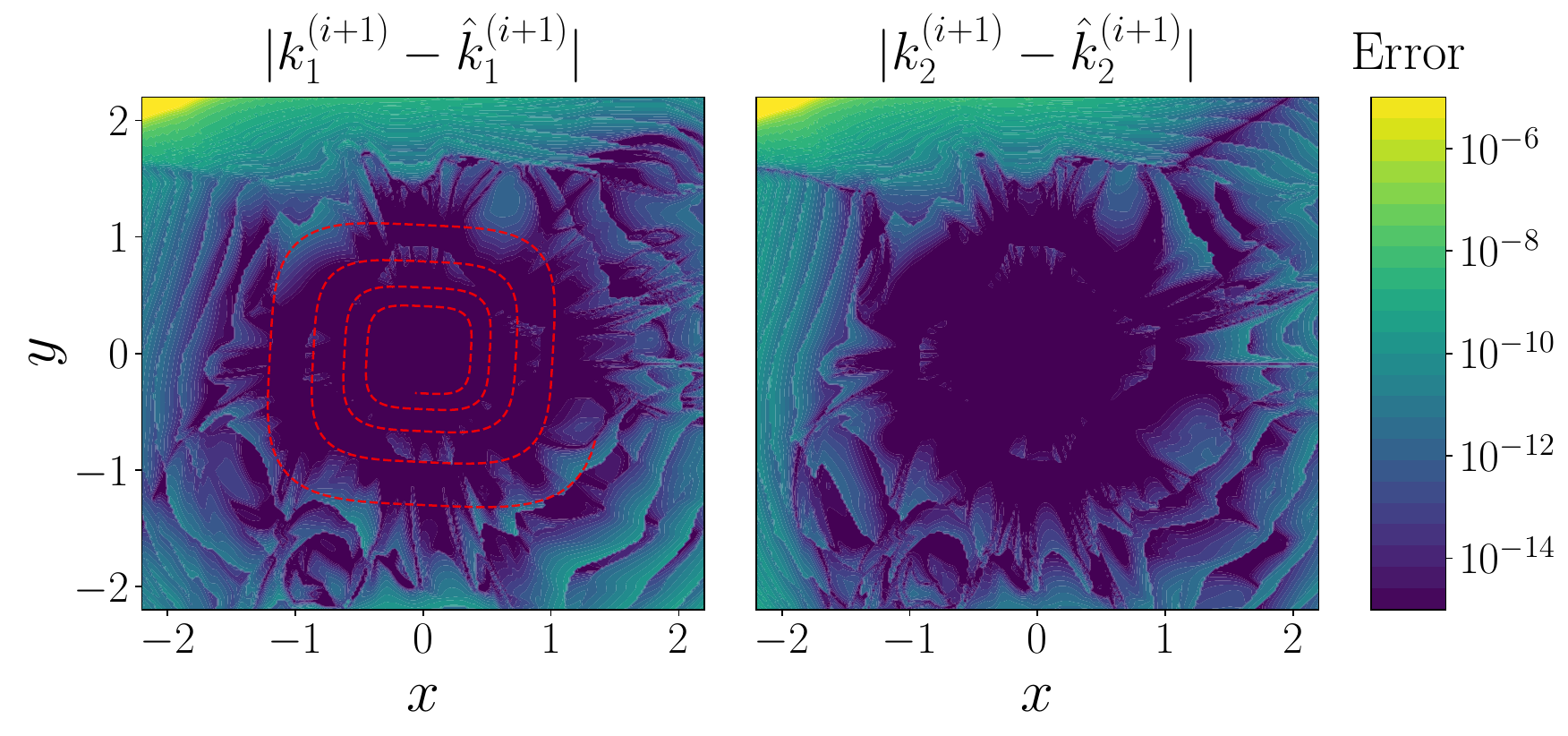}
    \caption{Error accumulated by \texttt{PRoNNS} trajectory steps in state space after training on the cubic oscillator. The red swirl shows the evolution of the \textit{training} trajectory which evolved at the farthest region of state space.}
    \label{fig:error_pronns}
\end{figure}

\subsection{Results: \texttt{CoNNS}}
We trained NNs containing four hidden layers with 40 (cubic), 50 (Hopf), and 100 (Kundur) nodes per layer. Notably, these NNs had to be several factors larger than the \texttt{PRoNNS} NNs in order to achieve acceptable error characteristics. We evaluated the performance of \texttt{CoNNS} by testing it on 50 trajectories stemming from the same distribution as the training set, and we bench-marked against an equivalently sized NN which was trained with no contraction constraints; both networks were trained to the same level of loss. Panels (a), (b), and (c) of Fig. (\ref{fig:trajectories}) compare the performance of \texttt{CoNNS} to both the unconstrained NN and Newton for all three dynamical systems. \texttt{CoNNS} clearly outperforms the vanilla NN in the case of the cubic and Kundur systems. In panel (d), where a larger perturbation is applied to the system, the unconstrained NN shows sharp, incoherent trajectories which are very poor predictors of the true underlying dynamics.

Table~\ref{tbl:error_conns} reports the cubic oscillator training and test results for both the constrained and unconstrained NNs.; similar statistics were computed for the Hopf and Kundur systems as well. Across all statistics, the results definitively suggest that \texttt{CoNNS} provides a more reliable prediction compared to the unconstrained NN predictions; furthermore, \texttt{CoNNS} required notably fewer iterations to converge. However, we required a relatively large NN to achieve these results (approximately 5x larger than \texttt{PRoNNS} in all cases), and the time domain error was relatively larger than the error accumulated by \texttt{PRoNNS}; compare, for instance, Fig. panel (c) of Fig~\ref{fig:trajectories_pronns} and panel (c) of Fig.~\ref{fig:trajectories}. %We further note, that for the cubic oscillator and Hopf bifurcation systems, we are able to achieve a level of accuracy with \texttt{CoNNS} that is equivalent to recent results reported in~\cite{Liu:2020}; however, the size of our NNs is approximately $3\times$ smaller.

\begin{figure}[ht]
    \centering
  \subfigure[\small Cubic Oscillator]{
       \includegraphics[width=0.32\linewidth]{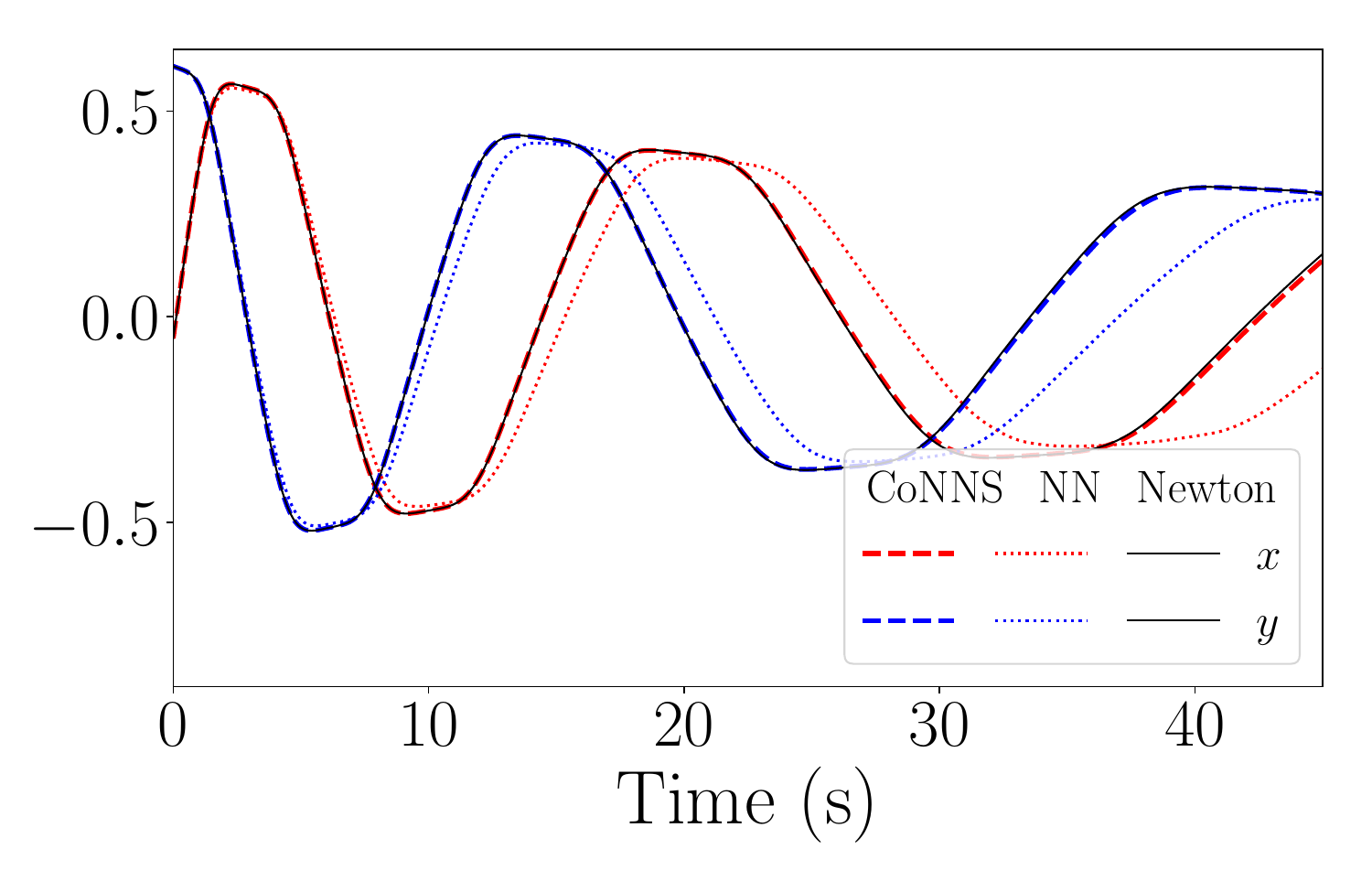}}
    \hfill
  \subfigure[\small Hopf Bifurcation]{
        \includegraphics[width=0.32\linewidth]{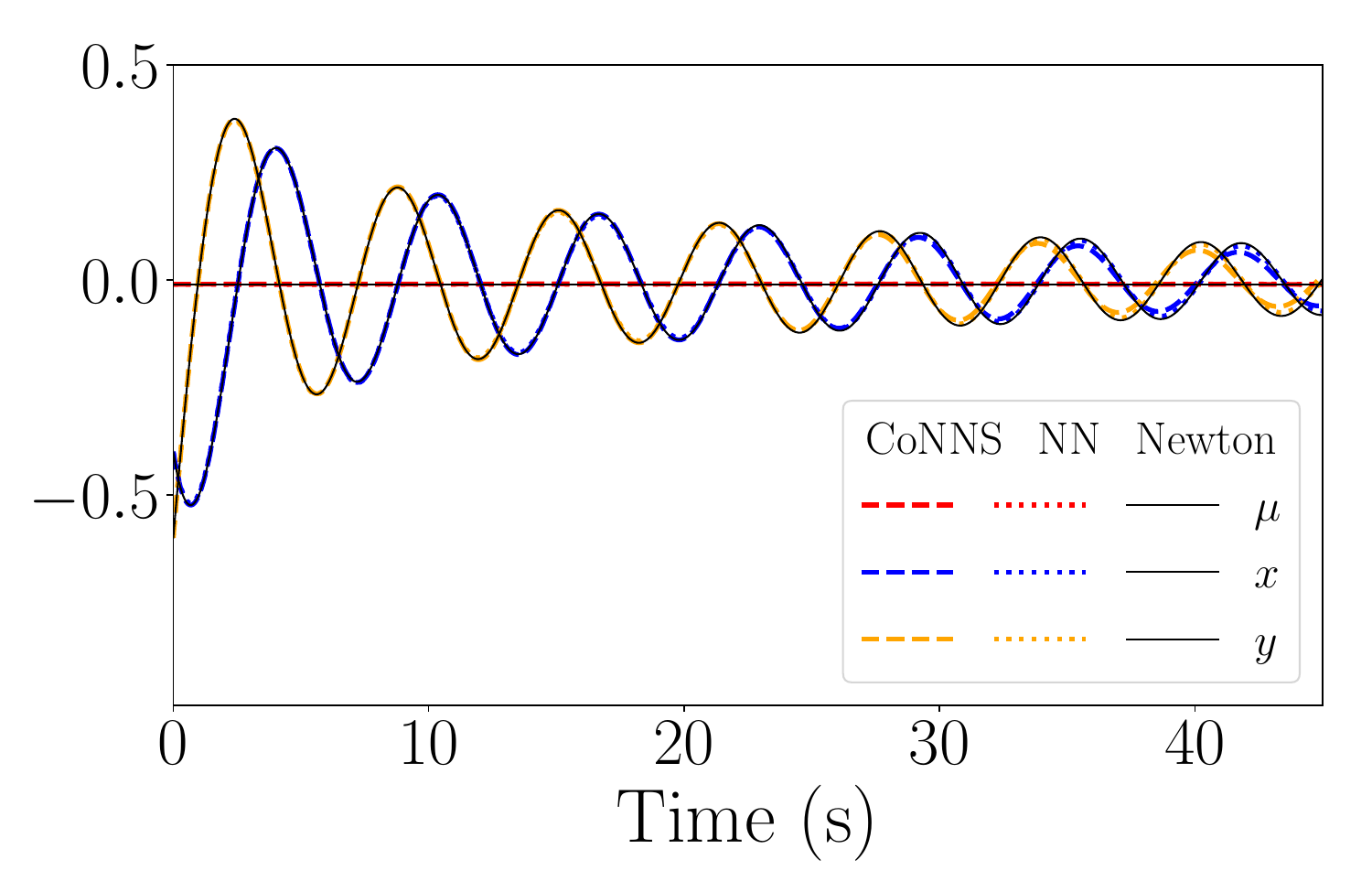}}
        \hfill
  \subfigure[\small Kundur System]{%
       \includegraphics[width=0.32\linewidth]{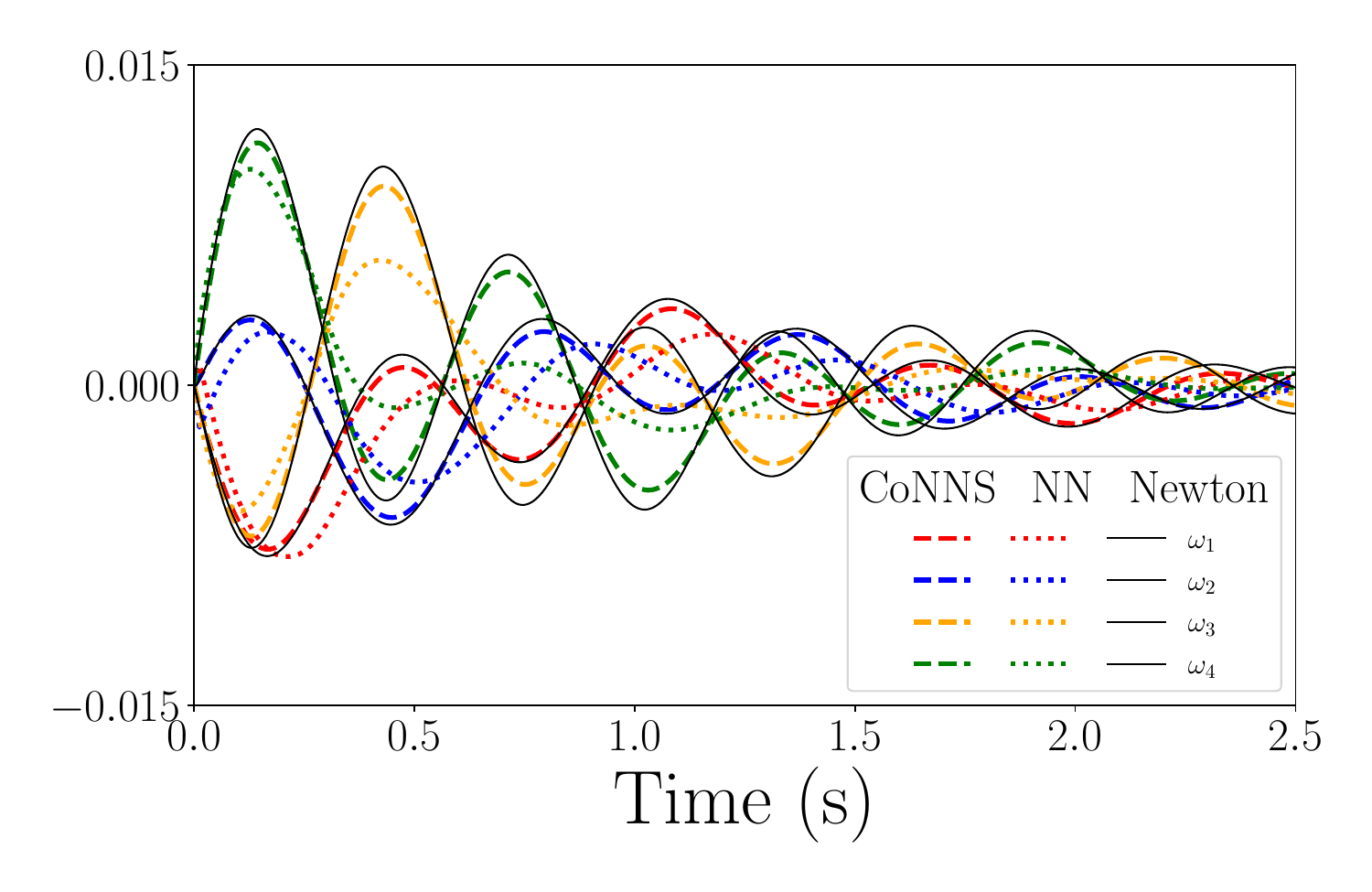}}\\
  \subfigure[\small Kundur: Larger Perturbation]{%
        \includegraphics[width=0.32\linewidth]{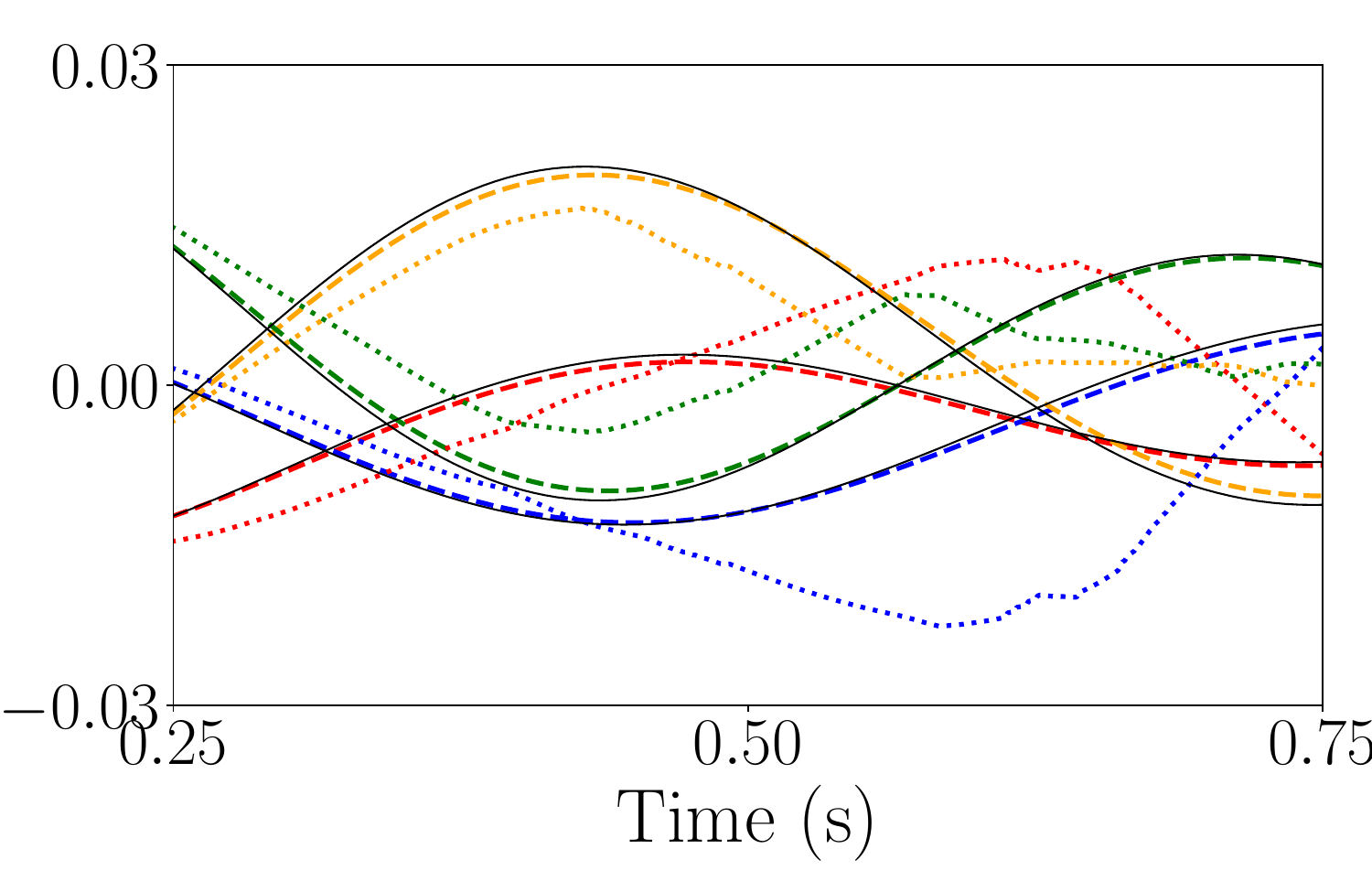}}
  \subfigure[\small \texttt{CoNNS}]{%
       \includegraphics[width=0.32\linewidth]{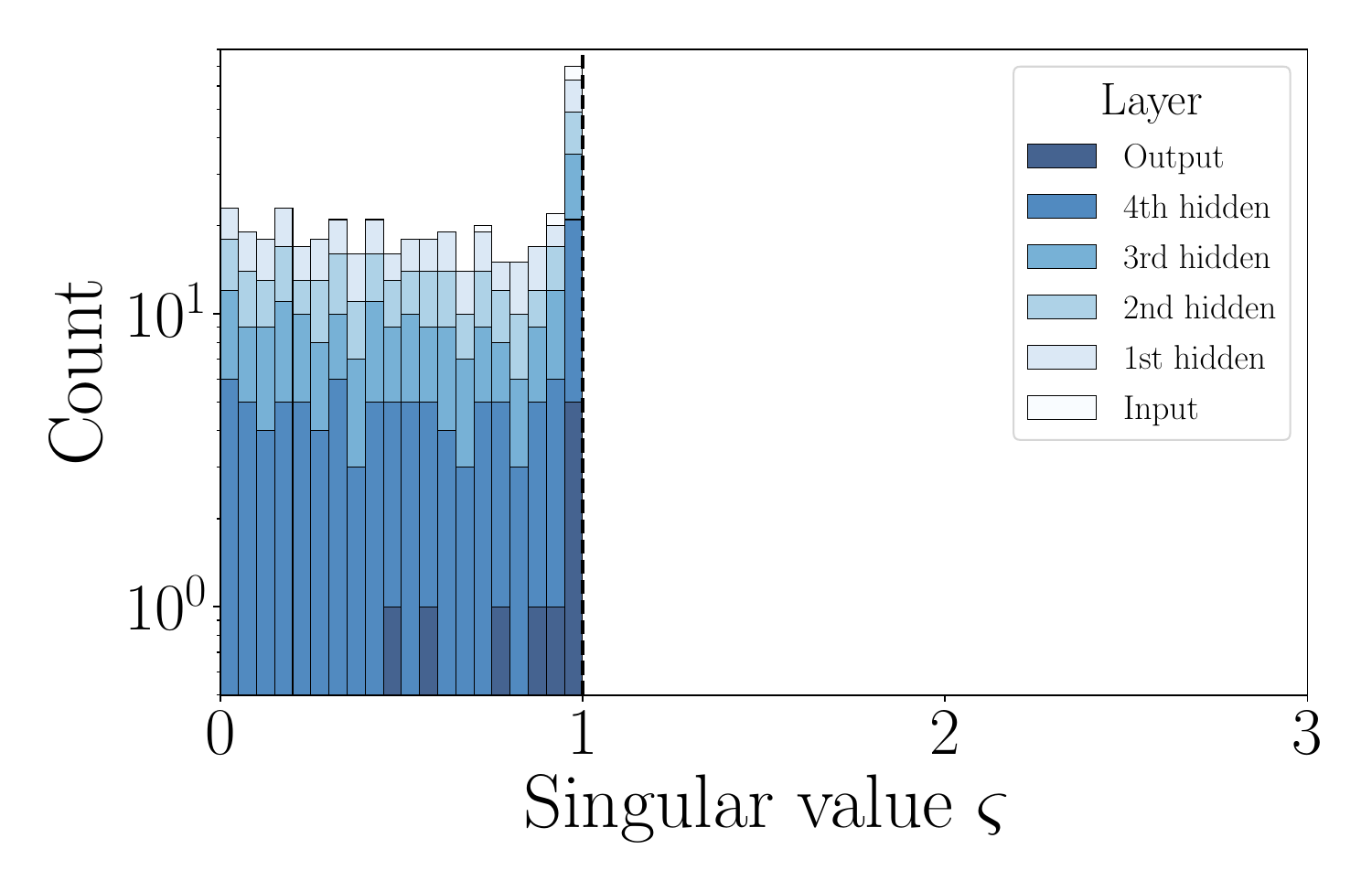}}
    \hfill
  \subfigure[\small Unconstrained NN]{%
        \includegraphics[width=0.32\linewidth]{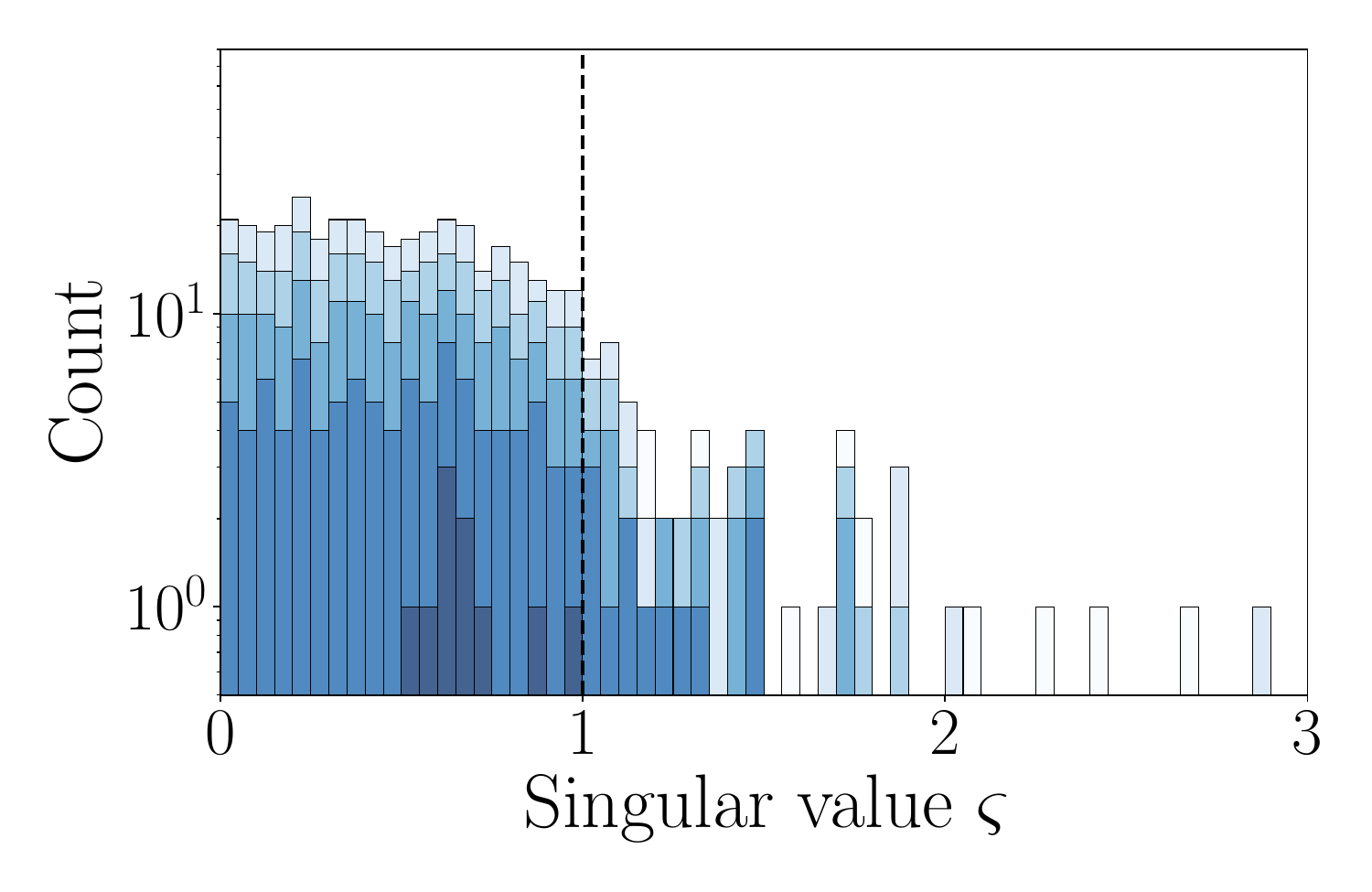}}
  \caption{Panels (a)-(d) show system trajectories simulated with Newton, \texttt{CoNNS} and an unconstrained NN. Singular values of the Kundur NN weighting matrices are shown in panel (e)-(f).}
  \label{fig:trajectories} 
\end{figure}

\begin{table}[ht]
  \caption{Training and test trajectory error. Iterations is the cumulative count of passes through \texttt{CoNNS} or the unconstrained NN for a given trajectory (Newton required a maximum of 8950 iterations, for reference).}
  \footnotesize
  \label{tbl:error_conns}
  \centering
  \begin{tabular}{llrrrrrrrr}
    \toprule
    & & \multicolumn{3}{c}{Cubic oscillator (\textit{constrained})} & \multicolumn{3}{c}{Cubic oscillator (\textit{unconstrained})} & \\
    \cmidrule(lr){3-5} \cmidrule(lr){6-8}
    Metric  & Data   &   $\left\Vert x - \hat{x} \right\Vert _{2}$ & $\left\Vert y - \hat{y} \right\Vert _{2}$ & iterations & $\left\Vert x - \hat{x} \right\Vert _{2}$ & $\left\Vert y - \hat{y} \right\Vert _{2}$ & iterations \\
    \midrule
    mean & Training & 0.3685 & 0.3620 & 12873 & 7.8412 & 7.9659 & 216531 \\
    % sd & Training & 0.3394 & 0.3442 & 1675 & 4.3590 & 4.5314 & 14528 \\
    % max & Training & 1.4939 & 1.4683 & 16044 & 16.7858 & 18.1205 & 236942\\
    \midrule
    mean & Test & 0.7297 & 0.7279 & 13177 & 8.7604 & 8.8905 & 213619  \\
    sd & Test & 0.8200 & 0.8230 & 1728 & 5.4595 & 5.2871 &  21309 \\
    max & Test & 4.0938 & 4.2462 & 16413 & 21.9004 & 22.5338 &  237974\\
    \bottomrule
  \end{tabular}
\end{table}

%\subsection{Discussion: computational complexity of Neural-Newton Solvers}\label{sub_sec: limits}
%For a dynamical system with $p$ states, the computational complexity of a single iteration of Newton's method is dominated by the linear system solve ${J}^{-1}{r}$, which is ${\mathcal O}(p^3)$, assuming the Jacobian function is available apriori. If the \texttt{CoNNS} NN is $h$ layers deep with $m$ nodes per layer, evaluation of the network will require $h$ matrix-vector products, for a total cost of ${\mathcal O}(h\cdot m^2)$. Since $m\gg h$, though, the cost is closer to $\sim{\mathcal O}(m^2)$. Because $\texttt{CoNNS}$ and Newton require a similar number of iterations to converge (at least, same order of magnitude), \texttt{CoNNS} will begin to offer computational benefits only when $p^{3/2}>m$. Thus, \texttt{CoNNS} will offer its primarily computational advantages in large systems, especially ones which evolve in low dimensional subspaces. One of the main limitations of \texttt{CoNNS}, however, comes from the computational cost of training with the SDP constraints, e.g., (\ref{eq: shcur_inexact}). For matrices that are ${\mathbb R}^{100\times100}$, iterative Adam steps slow down by almost an order of magnitude.

\section{Conclusion}

With the goal of accelerating time domain simulation speeds, this paper developed two learning-based methods for emulating numerical solvers. The Physics-pRojected Neural-Newton Solver (\texttt{PRoNNS}) modeled inverted Jacobian transformations, and it was found to very successfully emulate trapezoidal integration, accumulating very low degrees of trajectory error ($\le 10^{-3}$) across hundreds of tests, and achieving up to a 31\% speed-up over a Newton-based benchmark. This result was achieved on a relatively small, 10-state system, and computational advantages will certainly scale with system size. While no convergence guarantees were derived, \texttt{PRoNNS} was found to reliably converge to meaningful solutions. We also developed a fully data-driven solver, termed the Contracting Neural-Newton Solver (\texttt{CoNNS}), which incorporated training constraints to guarantee iterative convergence. \texttt{CoNNS} was found to consistently outperform a vanilla-NN benchmark, offering an order of magnitude more accurate test performance on both the cubic oscillator and Kundur systems. This level of accuracy, however, did not meet the level of \texttt{PRoNNS}, and it came at the cost of incorporating $5\times$ more hidden neurons, resulting in slower-than-Newton performance.

\bibliography{L4DC_bib}

\end{document}